\documentclass{article}

\usepackage{microtype}
\usepackage{graphicx}
\usepackage{subfigure}
\usepackage{booktabs} 

\usepackage{amsmath} 
\usepackage{amssymb}
\usepackage[utf8]{inputenc}
\usepackage{graphicx}
\usepackage{verbatim}
\usepackage{graphicx}
\usepackage{amsmath,amssymb,amsfonts,bm}
\usepackage{subfigure}
\usepackage{psfrag}
\usepackage[dvips]{epsfig}
\usepackage{color}
\usepackage[usenames,dvipsnames]{pstricks}
\usepackage[dvips]{epsfig}
\usepackage{pst-grad} 
\usepackage{pst-plot} 
\usepackage{enumerate}
\usepackage{amsbsy}
\usepackage{amsthm}
\usepackage{amscd}
\usepackage{subfigure}
\usepackage{color}
\usepackage{listings}
\usepackage{stackrel}
\usepackage{color} 
\usepackage{subfig}
\usepackage{graphicx}
\usepackage{flexisym}
\usepackage{setspace}
\usepackage{multirow}
\usepackage{mathtools}

\usepackage[algo2e,ruled,vlined,linesnumbered]{algorithm2e}
\SetInd{0em}{0.5em}
\SetKwComment{Comment}{/\!\!/}{}

\definecolor{mygreen}{RGB}{28,172,0} 
\definecolor{mylilas}{RGB}{170,55,241}

\newcommand{\ip}[2]{\langle #1, #2 \rangle}

\newtheorem{theorem}{Theorem}[section]

\newtheorem{definition}[theorem]{Definition}
\newtheorem{proposition}{Proposition}

\DeclareMathOperator{\soft}{soft}

\DeclareMathOperator{\shrink}{shrink}

\def\bA{{\bf A}}

\def\realR{{\mathbb{R}}}

\def\rank{{\rm rank}}

\def\Fig{Fig. }

\usepackage[accepted]{icml2019}

\icmltitlerunning{Masked Robust Principal Component Analysis}

\begin{document}

\twocolumn[
\icmltitle{Masked-RPCA: Sparse and Low-rank Decomposition Under Overlaying Model\\ and Application to Moving Object Detection }




\begin{icmlauthorlist}
\icmlauthor{Amirhossein Khalilian-Gourtani}{nyu}
\icmlauthor{Shervin Minaee}{expedia}
\icmlauthor{Yao Wang}{nyu}
\end{icmlauthorlist}

\icmlaffiliation{nyu}{Electrical and Computer Engineering Depatment, New York University. }
\icmlaffiliation{expedia}{Expedia Inc.}

\icmlcorrespondingauthor{Amirhossein Khalilian-Gourtani}{akg404@nyu.edu}

\icmlkeywords{Machine Learning, ICML}

\vskip 0.3in
]



\printAffiliationsAndNotice{}  


\begin{abstract}
Foreground detection in a given video sequence is a pivotal step in many computer vision applications such as video surveillance system. Robust Principal Component Analysis (RPCA) performs low-rank and sparse decomposition and accomplishes such a task when the background is stationary and the foreground is dynamic and relatively small. A fundamental issue with RPCA is the assumption that the low-rank and sparse components are added at each element, whereas in reality, the moving foreground is overlaid on the background. We propose the representation via masked decomposition (i.e. an overlaying model) where each element either belongs to the low-rank or the sparse component, decided by a mask.  We propose the Masked-RPCA algorithm to recover the mask and the low-rank components simultaneously, utilizing linearizing and alternating direction techniques. We further extend our formulation to be robust to dynamic changes in the background and enforce spatial connectivity in the foreground component. Our study shows significant improvement of the detected mask compared to post-processing on the sparse component obtained by other frameworks.
\end{abstract}

\section{Introduction}
Sparse and low-rank decomposition has been an active research area in signal and image processing in the past decade, with applications in motion segmentation \cite{TVRPCA}, \cite{ms2}, image foreground extraction \cite{ms3}, and optics \cite{4opt}.
In the simplest case, this problem can be formulated as:
\begin{equation}
\underset{L,S}{\text{minimize}} \quad \text{rank}(L)+ \lambda_s \|S\|_0 \quad \text{s.t.}\quad X= L+  S,
\label{Eq_RPCA}
\end{equation}
where $L$ and $S$ denote the low-rank and sparse components of the signal $X$, respectively.
There are some situations in which a unique decomposition may not exist; e.g. if the low-rank matrix L itself is also very sparse, it becomes very hard to uniquely identify it from another sparse matrix.
Therefore, there have been many studies to find the conditions under which this decomposition is possible, such as the works in \cite{OriginalRPCA}, \cite{6rpca}.
Also because of the non-convexity of both the rank function and the $\ell_0$ norm, the problem in \eqref{Eq_RPCA} is NP-hard. In order to be able to solve this decomposition, usually the $\text{rank}(L)$ is relaxed to $\| L\|_*$ (the nuclear norm of $L$, which is the sum of its singular values), and the $\|S\|_0$ is relaxed by the $\|S\|_1$ approximation \cite{NucNormInsteadOfRank}.

The dominant application of sparse and low-rank decomposition (aka RPCA) has been for moving object detection in videos \cite{sam1}, \cite{sam2}, but it has also been used for various other applications.
To name some of the prominent works, in \cite{8peng}, Peng et al proposed a sparse and low-rank decomposition approach with application for robust image alignment.
A similar approach has been proposed by Zhang \cite{9tilt} for transform invariant low-rank textures.
In \cite{mon}, Keshavan proposed an algorithm for matrix completion using low-rank decomposition.

There has been several improvement of the vanilla RPCA over the past decade, and despite their great improvements in terms of accuracy and speed, there is a fundamental limitation in most of these models. The basic assumption that all these models share is the additive model for the sparse and low-rank components. In reality, the dynamic foreground object is overlaid on top of the low-rank background. For a more detailed overview of RPCA extensions, we refer the readers to \cite{overview1}.

In this work, we try to address this issue by  assuming a model for the case where the two components are overlaid on top of each other (instead of simply being added). Thus, each element of $X$ comes only from one of the components. Therefore, besides deriving the sparse and low-rank component we need to find their supports. Assuming $W \in \left\{0,1\right\}^{mn\times k}$ denotes the support of $S$, we can write this overlaid signal summation as $X= (1-W) \circ L+ W \circ S$. We can separate these components by assuming some prior knowledge on $L$, $S$ and $W$ terms, and forming an optimization problem.
In fact, we do not need to even include the $S$ term in our optimization framework, since by having $W$, the $S$ component can easily be derived as $S= W \circ X$. We propose an optimization algorithm (to be called M-RPCA) based on the alternating direction method of multipliers (ADMM) \cite{boydADMM} and ideas of linearizing \cite{LinearizedADMM}. We show the convergence of the proposed algorithm to a Karush–Kuhn–Tucker (KKT) point under reasonable assumptions. Our experiments show that the proposed framework directly recovers the mask of the foreground without need for post processing on the sparse component as in the RPCA algorithm. 

As with the original RPCA algorithm, the proposed M-RPCA algorithm has two limitations: 1) It does not enforce spatial connectivity of the foreground, and 2) when the background is not stationary and has random perturbations (such as water waves, moving leaves, etc.), these perturbations are usually picked up by the sparse component, leading to noisy foreground detection. We further show extensions of the proposed framework to tackle these problems. Following the idea of \cite{TVRPCA}, we model  the background as the sum of  a low rank component and a sparse component (used to model the random perturbation in the dynamic background), and furthermore enforce the spatial connectivity of the foreground object by adding a total variation penalty on the mask in the optimization formulation. We propose an optimization algorithm (to be called extended M-RPCA or EM-RPCA) to solve for all three components and show that it leads to significant improvement over M-RPCA in sequences with dynamic background.

The idea of solving a masked decomposition problem was first proposed in \cite{shervin} for image segmentation, where an image is considered to have two overlaid  components (e.g. text overlaid on background), each modeled by a subspace.  Here, we extend this work by assuming one component is low-rank, while the other is sparse.




The structure of the rest of this paper is as follows: Section II presents the problem formulation, and the proposed optimization framework to solve it, as well as a convergence analysis. 
Section III provides the detailed experimental results of the proposed framework for moving object detection, and its comparison with previous state-of-the-arts models. 
And finally the paper is concluded in Section V.
\section{Problem Formulation and Solution}
In this section we introduce the general framework of masked robust principal component analysis (Masked-RPCA) formulated as an optimization problem, propose an algorithmic solution based on ADMM and linearizing techniques, and investigate the convergence properties. 

\subsection{Masked Robust Principal Component Analysis}
 Given a sequence of video frames in $X^{3d} \in \realR^{m\times n\times k}$ let us denote the matrix $X \in \realR^{mn\times k}$ which is constructed by vectorizing and stacking the frames of the video. Then, the goal is to recover the matrices $L\in\realR^{mn\times k}$ and $W\in\left\{0,1\right\}^{mn\times k}$ such that $W$ denotes the foreground support and the low-rank matrix $L$ matches the video sequence $X$ wherever the foreground is not active. A plausible formulation of such problem can be written as \eqref{Eq_MCA_binaryW}.
\begin{equation}
\begin{aligned}
& \underset{L,W}{\text{minimize}}
& &\rank(L)+ \lambda_w \psi(W) \\
& \text{subject to:}
& & (1-W) \circ \left(X-L\right) = 0\\
& & & W \in \left\{0,1\right\}^{mn\times k}
\end{aligned}
\label{Eq_MCA_binaryW}
\end{equation}
where $\psi(\cdot)$ encodes our prior knowledge about $W$ and $\lambda_w\in \realR$ is the regularization parameter. The problem as stated in \eqref{Eq_MCA_binaryW} is not tractable because 1) the general rank minimization problem is NP-hard \cite{NucNormInsteadOfRank}, and 2) recovering the binary matrix $W$ requires solving a combinatorial problem. To manage the rank minimization in general, nuclear norm minimization is proposed as a surrogate especially in the context of matrix completion \cite{NucNormInsteadOfRank,OriginalRPCA}. Additionally, the $ W \in \left\{0,1\right\}^{mn\times k}$ constraint can be relaxed to the convex interval between zero and one, namely, $W \in \left[0,1\right]^{mn\times k}$. Imposing the sparsity of desired $W$ via $\ell_1$-norm we can formulate the problem as in \eqref{Eq_Main_Formulation}.
\begin{equation}
\begin{aligned}
& \underset{L,W}{\text{minimize}}
& & \|L\|_*+ \lambda_w \|W\|_{1} \\
& \text{subject to:}
& & (1-W) \circ \left(X-L\right) = 0\\
& & & W \in \left[0,1\right]^{mn\times k}
\end{aligned}
\label{Eq_Main_Formulation}
\end{equation}
The algorithm for solving the problem in \eqref{Eq_Main_Formulation} is not immediately apparent especially since the variables $W$ and $L$ are coupled. For general low-rank and sparse decomposition formulations the ADMM algorithm is shown to be effective \cite{OriginalRPCA}. More recently, ADMM for multi-affine constraints under certain assumptions was introduced and analyzed \cite{ADMMFM2018}. Additionally, ideas of linearizing such as Linearized Alternating Direction Method (LADM) for general affine constraint \cite{LinearizedADMM} and  for nuclear norm minimization\cite{LADMMforNuclear} were introduced to handle more complicated affine constraints. Here, we propose to use the linearizing techniques for the bi-affine constraint as in \eqref{Eq_Main_Formulation}. This way not only we can deal with the coupling of the variables but also we will find closed form solution for each sub-problem of the ADMM algorithm.

In the following section we drive the steps of the algorithm by forming the augmented Lagrangian and minimizing the linearized augmented Lagrangian w.r.t. each variable. Let us denote the dual variable for the equality constraint by $U_x$ and abuse notation to show the indicator function over each element of matrix $W$ by $\iota_{\left[0,1\right]}(W)$ where $\iota_{\Omega}(x) $ takes the value 0 if $x\in\Omega$, otherwise infinity. The augmented Lagrangian can be written as in \eqref{Eq_LagrangianP1}.
\begin{equation}
\begin{split}
\mathcal{L} (L,W,U_x) &= \|L\|_*+ \lambda_w \|W\|_{1}+ \iota_{\left[0,1\right]}\left(W\right)\\&+\ip{U_x}{(1-W)\circ(L-X)}\\&+\frac{\rho_x}{2} \|(1-W) \circ L- (1-W) \circ X\|^2
\end{split}
\label{Eq_LagrangianP1}
\end{equation}

\begin{definition}
For simplicity, let us define the following notation, where superscript $i$ denotes the iteration number.
\begin{equation*}
\Lambda^{i}_L \triangleq  \left(1-W^i\right)\circ \left((L^i-X)\circ(1-W^i)+\frac{U_x^i}{\rho_x}\right)
\end{equation*}
\begin{equation*}
\Lambda^{i}_W \triangleq  \left(X-L^{i+1}\right)\circ \left((L^{i+1}-X)\circ(1-W^i)+\frac{U_x^i}{\rho_x}\right)
\end{equation*}
\end{definition}
\begin{definition}
Given matrix $Y\in\realR^{m\times n}$ and $\delta>0$, let $Y = U\Sigma V^T$ and $I$ the identity matrix then, $\mathcal{D}(Y,\delta) = U(\Sigma - \delta I)_+V^T$ where $(a)_+ = \max\{0,a\}$ denotes the singular value thresholding operator. 
\end{definition}
\begin{definition}
  $\Pi_{\left[0,1\right]}$ denotes The projection onto the interval $\left[0,1\right]$.
\end{definition}
The update for $L$ at each iteration is achieved by minimizing the linearized augmented Lagrangian while fixing the variables $W^i$ and $U_x^i$:
\begin{equation*}
\begin{split}
L^{i+1} = \arg\min_{L}\left\|L\right\|_*+\frac{\rho_x}{2} \left\|(1-W^i)\circ(L-X)+\frac{U_x^i}{\rho_x}\right\|^2
\end{split}
\end{equation*}
We can linearize the quadratic term as
\begin{equation*}
\begin{split}
&\frac{1}{2} \left\|(1-W^i)\circ(L-X)+\frac{U_x^i}{\rho_x}\right\|^2\simeq \frac{1}{2\tau_L} \|L-L^i\|^2\\&+\ip{\Lambda_L^i}{L-L^i}+ \frac{1}{2} \left\|(L^i-X)\circ(1-W^i)+\frac{U_x^i}{\rho_x}\right\|^2
\end{split}
\end{equation*}
where $\tau_L>0$ is the proximal parameter. 
As a result the update rule for $L$ can be written as in \eqref{Eq_updateL_P1}.
\begin{equation}
\begin{split}
L^{i+1} &= \arg\min_{L} \|L\|_*+ \frac{\rho_x}{2\tau_L} \left\|L - \left(L^i - \tau_L  \Lambda_L^{i}\right)\right\|^2\\ &= \mathcal{D}\left(L^i - \tau_L  \Lambda_L^{i}, \frac{\tau_L}{\rho_x}\right)
\end{split}
\label{Eq_updateL_P1}
\end{equation}

The update rule for $W$ is achieved by minimizing the linearized augmented Lagrangian while fixing $L^{i+1}$ and $U^i_x$. Using the same technique we have
\begin{equation*}
\begin{split}
&\frac{1}{2} \left\|(1-W)\circ(L^{i+1}-X)+\frac{U_x^i}{\rho_x}\right\|^2\simeq \frac{1}{2\tau_L} \|W-W^i\|^2\\&+\ip{\Lambda_W^i}{W-W^i}+ \frac{1}{2} \left\|(L^{i+1}-X)\circ(1-W^i)+\frac{U_x^i}{\rho_x}\right\|^2
\end{split}
\end{equation*}
where $\tau_W\geq0$ is the proximal parameter.
As a result, the update rule for $W$ can be written as
\begin{equation}
\begin{split}
W^{i+1} &= \arg\min_{W}\lambda_w\|W\|_{1}+\iota_{\left[0,1\right]}(W) \\&+ \frac{\rho_x}{2\tau_W} \left\|W - \left(W^i - \tau_W  \Lambda_W^{i}\right)\right\|^2\\ &= \Pi_{\left[0,1\right]}\left[\soft\left(W^i - \tau_W  \Lambda_W^{i}, \frac{\lambda_w\tau_W}{\rho_x}\right)\right]
\end{split}
\label{Eq_updateW_P1}
\end{equation}

The update rule for $U_x$ is done by dual ascent as
\begin{equation}
U_x^{i+1} = U_x^i + \rho_x \left((1-W^{i+1})\circ(L^{i+1}-X)\right)
\label{Eq_updateU_P1}
\end{equation}
The steps of the algorithm are summarized in Alg. \ref{Alg_ADMMforMRPCA}.
\begin{algorithm2e}[h]
	\SetAlgorithmName{Alg.}{}
	
	\caption{ADMM with linearizing applied to Eq.~\eqref{Eq_Main_Formulation} }
	 \setstretch{1.1}
	\label{Alg_ADMMforMRPCA}	
	\begin{small}
		\textbf{Input:} $X$, $\lambda_w$,$\rho_x$\;
		$L\gets \text{median}(X)$\quad $W \gets 0$ \quad $U_x\gets 0$ \;
		\While{not converged}{
			\Comment{Main operations detailed in comments}
			\Comment{Singular Value Thresholding  \eqref{Eq_updateL_P1}}
			$\displaystyle L \gets \arg \min_{{\bf A}} \widehat{\mathcal{L}}_L({\bf A}, W,U_x)$ 
			
			\Comment{Soft-thresholding \& projection \eqref{Eq_updateW_P1}}
			$\displaystyle W \gets \arg \min_{{\bf A}} \widehat{\mathcal{L}}_W (L, {\bf A}, U_x)$
			
			\Comment{Element-wise mult and add \eqref{Eq_updateU_P1}}
			$U_x \gets U_x + \rho_x \left((1-W) \circ \left(X-L\right)\right)$\;
		}
		\textbf{Output:} $L,~W$\;
	\end{small}
\end{algorithm2e}
\subsubsection{Convergence Analysis}
In this section, we state and prove results regarding the convergence analysis of the proposed algorithm for solving \eqref{Eq_Main_Formulation}. 
\begin{proposition}
Linear independence constraint qualification (LICQ) holds for the problem in \eqref{Eq_Main_Formulation}.
\end{proposition}
\begin{proposition}
	\label{Prop_1}
      Denote the variable at next iteration by superscript $+$ then, the update of the dual variable $U_x$ increases the augmented Lagrangian such that 
	\begin{equation*}
	\mathcal{L}(L^+,W^+,U_x^+) - \mathcal{L}(L^+,W^+,U_x) = \frac{1}{\rho_x} \left\|U_x^+ - U_x\right\|
	\end{equation*}
\end{proposition}
\begin{proof}
This can be proved using the augmented Lagrangian in \eqref{Eq_LagrangianP1} and the update rule for $U_x$. The proof is provided in the supplementary material.
\end{proof}
\begin{proposition}\label{Prop_2}
	Suppose that $\left\| U_x^{i+1} - U_x^{i}\right\| \rightarrow 0$. Then, $\nabla_{U_x} \mathcal{L}(L,W,U_x) \rightarrow 0$ and every limit point of the sequence $\{(L^i,W^i)\}_0^\infty$ is feasible. 
\end{proposition}
\begin{proof}
The proof is provided in the supplementary material.
\end{proof}
\begin{proposition}
	\label{Prop_3} Let $\partial f(x)$ denote the general subdifferential of $f$ at $x$ \cite{rockafellar2009variational}, then
	\begin{equation}
	\begin{split}
	-\frac{\rho_x}{\tau_L} \left(L^{i+1}-L^{i}\right) - \rho_x \Lambda_L^i \in \partial f(L^{i+1})
	\end{split}
	\end{equation}
		\begin{equation}
	\begin{split}
	-\frac{\rho_x}{\tau_W} \left(W^{i+1}-W^{i}\right) - \rho_x \Lambda_W^i \in \partial g(W^{i+1})
	\end{split}
	\end{equation}
where $f(L) = \left\|L\right\|_*$ and $g(W) =\lambda_w \|W\|_{1}+ \iota_{\left[0,1\right]}\left(W\right) $.
\end{proposition}
\begin{proof}
The two statements can be checked from the optimality conditions of \eqref{Eq_updateL_P1} and \eqref{Eq_updateW_P1}.
\end{proof}
\begin{proposition}
Assume the sequence $\left\{\left(L^i,W^i,U_x^i\right)\right\}_0^\infty$ is bounded, then every limit point $\left(L^\infty,W^\infty,U_x^\infty\right)$ is a Karush–Kuhn–Tucker (KKT) point of \eqref{Eq_Main_Formulation}.
\end{proposition}
\begin{proof}
 $\left\{\left(L^i,W^i,U_x^i\right)\right\}_0^\infty$ is bounded, hence from Proposition \ref{Prop_2} every limit point $\left(L^\infty,W^\infty\right)$ satisfies $(1-W^\infty)\circ(L^\infty-X) = 0$. Additionally, by Proposition \ref{Prop_3} and the definition of the general sub-gradient \cite{rockafellar2009variational}, we get
 $f(L)\geq f(L^i) + \ip{L-L^i}{-\frac{\rho_x}{\tau_L} \left(L^{i}-L^{i-1}\right) - \rho_x \Lambda_L^{i-1}}
 ~~\forall L$. Letting $i\rightarrow \infty$ and using the limit point feasibility we get $f(L)\geq f(L^{\infty})+\ip{L-L^\infty}{-(1-W^\infty)\circ U_x^\infty}~~\forall L$ resulting in $-(1-W^\infty)\circ U_x^\infty\in \partial f(L^\infty)$. Similarly, we can show that $-(L^\infty-X)\circ U_x^\infty \in \partial g(W^\infty)$, hence concluding the KKT condition for the limit points.
\end{proof}

\section{Extension of M-RPCA }
In this section, we further extend the proposed formulation to handle more challenging scenarios such as cases where changes in the background are present. Our goal is to consider more realistic and challenging scenarios, that dynamic background is present and extend the framework to tackle such cases. Additionally, we enforce spacial and temporal connectivity of the foreground mask.

Let us consider a case where dynamic changes are present as part of the background (such as the motion of leaves in the wind). In the original RPCA such changes would be considered as dynamic perturbations to the static background and will be separated to the sparse component which in turn would result in the noise appearing on the mask after thresholding the sparse component.

The formulation in \eqref{Eq_Main_Formulation} is prone to a  similar problem as the RPCA. Let us consider the presence of dynamic changes in the background and the equality constraint in \eqref{Eq_Main_Formulation}. Depending on the value of $\lambda_w$ in \eqref{Eq_Main_Formulation}, the contribution of these small changes is either considered as part of the foreground mask or will stay present on the estimated background. To further explain, consider large value of $\lambda_w$ and severe random noise as the perturbation. In such case, the cost of adding extra pixels to the mask is high and since the equality constraint has to be satisfied, it would be plausible to accept the noise term on the $L$ variable (which results in a slightly higher value for the nuclear norm) to satisfy the constraint. On the other hand, for small values of $\lambda_w$ the noise will be picked by the foreground mask (satisfying the constraint) which is not appealing. 

In order to extend the formulation such that it can handle dynamic background and is robust to noise, we can slightly change the equality constraint and use more regularization terms. Our prior knowledge about the foreground mask is that it contains more or less connected components as opposed to the dynamic background which has a more random and sparse nature. As a result, we can change the equality constraint to be $(1-W)\circ(X-L) = E$ where $E$ is assumed to be the sparse perturbation. In this case, $L$ does not have to match $X$ where ever $W$ is nonzero and the difference will be considered as part of $E$. We further need to enforce our prior knowledge (sparsity of $E$ and spacial connectivity of $W$) as regularization terms in the objective function. Let us first define some notation.
\begin{definition}
	For a 3D matrix $T\in\realR^{m\times n\times k}$ let $$\left\||DT|\right\|_1 \triangleq \left\| ((D_hT)_{ijk}^2 + (D_v T)_{ijk}^2 + (D_d T)_{ijk}^2)^{1/2} \right\|_1$$ with $D_h,D_v,D_d$ being the horizontal, vertical, and depth derivative operators, respectively, then $\left\||DT|\right\|_1$ denotes the total variation (TV-norm) of the matrix $T$ \cite{norm3Dl1l2}.
\end{definition}
\begin{definition}
	For a 3D matrix $T\in\realR^{m\times n\times k}$ let $\mathcal{R}(T) \triangleq [\text{vec}(T_{\cdot\cdot1}),\cdots,\text{vec}(T_{\cdot\cdot k})]$ denote the reshape operator from a 3D matrix to a 2D matrix by stacking the frames as columns and the $\mathcal{R}^{-1}(\cdot)$ as the inverse operator. 
\end{definition}
In order to enforce spatial connectivity of the foreground mask, we can regularize the total variation (TV-norm) of estimated foreground mask. Noting that the TV-norm only regularizes the changes, we would also like to regularize the total energy of the mask such that ideally the estimated mask would be piece-wise constant with most values set to zero. In order to enforce that, we add the sparsity of the noise contribution through  $\ell_1$-norm. As a result, the problem can be formulated as in \eqref{Eq_preExtended}.
\begin{equation*}
\begin{aligned}
& \underset{L,W,E}{\text{minimize}}
&& \|L\|_*+ \underbrace{\lambda_e \left\|E\right\|_1}_{\text{sparse noise}} \\
&&&+ \underbrace{\lambda_w \|W\|^2 + \lambda_z \left\|\left|D\mathcal{R}^{-1}(W)\right|\right\|_1}_{\text{enforce connectivity of the foreground}} \\
&\text{subject to:}
&& (1-W) \circ \left(X-L\right) = E,~~ W \in \left[0,1\right]^{mn\times k}
\end{aligned}
\label{Eq_preExtended}
\end{equation*}
Adding another extra variable $Z$ to this formulation makes the solution via ADMM with linearizing techniques possible. 
\begin{equation}
\begin{aligned}
& \underset{L,W,Z,E}{\text{minimize}}
& & \|L\|_*+ \lambda_w \|W\|^2 + \lambda_z \left\||Z|\right\|_1 + \lambda_e \left\|E\right\|_1 \\
& \text{subject to:}
& & (1-W) \circ \left(X-L\right) = E\\
& & & D_{\text{3D}}\left(W\right) = Z, \quad W \in \left[0,1\right]^{mn\times k}
\end{aligned}
\label{Eq_Extended_Formulation}
\end{equation}
where $D_{3D}(\cdot) = D\mathcal{R}^{-1}(\cdot)$.

In the following part we derive the algorithm for solving \eqref{Eq_Extended_Formulation} using the ADMM approach with linearizing with respect to the coupled variables. First, we form the augmented Lagrangian by introducing the dual variables $U_x$ and $U_z$ for the equality constraints as in \eqref{Eq_Lagrangian_Extended}. Then, we minimize the augmented Lagrangian w.r.t. each variable while keeping the others fixed and linearizing the quadratic term with the coupled variables. 
\begin{align}
&\mathcal{L}(L,W,Z,E,U_x,U_z) =  \|L\|_* + \lambda_z \left\||Z|\right\|_1 + \lambda_e \left\|E\right\|_1\notag\\ &+ \lambda_w \|W\|^2 + ~\iota_{\left[0,1\right]}(W) + \ip{U_z}{Z-D_{\text{3D}}(W)} \notag\\&+ \frac{\rho_z}{2}\left\|Z-D_{\text{3D}}(W)\right\|^2+ \ip{U_x}{(1-W)\circ(L-X)+E}\notag\\&+ \frac{\rho_x}{2}\left\|(1-W)\circ(L-X)+E\right\|^2
\label{Eq_Lagrangian_Extended}
\end{align}
\begin{definition}
	for simplicity of the notation let us define 
	\begin{small}
		$\Psi^{i}_L\triangleq\left(1-W^i\right)\circ \left((L^i-X)\circ(1-W^i)+E^i+\frac{U_x^i}{\rho_x}\right)$ and 
		$\Psi^{i}_W \triangleq  \left(X-L^{i+1}\right)\circ\left[(L^{i+1}-X)\circ(1-W^i)+E^i+\frac{U_x^i}{\rho_x}\right]$ and $\widehat{\Psi}_W^i \triangleq W^i - \tau_W \Psi_W^i$.
	\end{small}
\end{definition}
The update rule for $L$ has the same form as the previous section so it can be linearized by following similar steps and we get,
\begin{equation}
\begin{split}
L^{i+1} &= \arg\min_{L} \|L\|_*+ \frac{\rho_x}{2\tau_L} \left\|L - \left(L^i - \tau_L  \Psi_L^{i}\right)\right\|^2\\ &= \mathcal{D}\left(L^i - \tau_L  \Psi_L^{i}, \frac{\tau_L}{\rho_x}\right)
\end{split}
\label{Eq_updateL_P2}
\end{equation}

The update rule for $W$ can be written as
\begin{equation}
\begin{split}
&W^{i+1} = \arg\min_{W} \lambda_w \|W\|^2 + ~\iota_{\left[0,1\right]}(W)\\&+ \frac{\rho_z}{2}\left\|Z^i-D_{\text{3D}}(W)+\frac{U_z^i}{\rho_z}\right\|^2\\&+ \frac{\rho_x}{2}\left\|(1-W)\circ(L^{i+1}-X)+E^i + \frac{U_x^i}{\rho_x}\right\|^2
\end{split}
\label{Eq_updateWnotLinear}
\end{equation}
and by linearizing the last term in the right hand side of \eqref{Eq_updateWnotLinear} we get the following minimization problem.
\begin{align}
&W^{i+1} = \arg\min_{W} \lambda_w \|W\|^2 + ~\iota_{\left[0,1\right]}(W)\\& +\frac{\rho_z}{2}\left\|Z^i-D_{\text{3D}}(W)+\frac{U_z^i}{\rho_z}\right\|^2 + \frac{\rho_x}{2}\left\|W-\Psi^{i}_W \right\|^2\notag
\label{Eq_updateWLinear}
\end{align}
The computationally efficient solution can be achieved by employing the Fourier transform and solving a diagonal system of equations (element-wise devision). 
\begin{equation}
\begin{split}
W^{i+1} = \Pi_{\left[0,1\right]}\left[\mathcal{R}\mathcal{F}^{-1}\frac{\mathcal{F}\left(\mathcal{R}^{-1}\left(\Gamma \right)\right)}{\left(\alpha I + \rho_z \Sigma_D^H\Sigma_D\right)}\right]
\end{split}
\label{Eq_updateW_P2}
\end{equation}

where $\alpha =2\lambda_w+\rho_x/\tau_W$, $\mathcal{F}$ and $\mathcal{F}^{-1}$ are the 3D Fourier transform pair, $\Sigma_D = \mathcal{F}^{-1}D\mathcal{F}$, and $\Gamma = \frac{\rho_x}{\tau_W} \widehat{\Psi}_W^{i} + \rho_z \mathcal{R}\left(D^T(Z+U_z/\rho_z)\right)$.

The update rule for $E$ is achieved by minimizing the augmented Lagrangian w.r.t. $E$ which results in element-wise soft-thresholding as
\begin{align}
\label{Eq_updateS_P2}
E^{i+1} &= \arg\min_{E} \lambda_e \left\|E\right\|_1\\& +\frac{\rho_x}{2}\left\|E+(1-W^{i+1})\circ(L^{i+1}-X) + \frac{U_x^i}{\rho_x}\right\|^2\nonumber\\&=\soft\left((W^{i+1}-1)\circ(L^{i+1}-X) - \frac{U_x^i}{\rho_x},\frac{\lambda_e}{\rho_x}\right)\nonumber
\end{align}
The update rule for $Z$ is achieved by solving the minimization problem as in \eqref{Eq_updateZ_P2} where the proximal operator is denoted by $\shrink(\cdot)$ \cite{norm3Dl1l2}.
\begin{align}
Z^{i+1} &= \arg\min_{Z} \lambda_z \left\||Z|\right\|_1  + \frac{\rho_z}{2}\left\|Z-D_{\text{3D}}(W^{i})+\frac{U_z^i}{\rho_z}\right\|^2\nonumber\\&=\shrink\left(D_{\text{3D}}(W^{i+1})-\frac{U_z^i}{\rho_z},\frac{\lambda_z}{\rho_z}\right)
\label{Eq_updateZ_P2}
\end{align}
The dual variables are updated according to the equality constraints as in \eqref{Eq_updateUx_P2} and \eqref{Eq_updateUz_P2}.
\begin{equation}
U_x^{i+1} = U_x^{i} + \rho_x \left((1-W^{i+1}\circ(L^{i+1}-X)+E^{i+1})\right)
\label{Eq_updateUx_P2}
\end{equation}
\begin{equation}
U_z^{i+1} = U_z^{i} + \rho_z \left(Z^{i+1}-D_{\text{3D}}(W^{i+1})\right)
\label{Eq_updateUz_P2}
\end{equation}
\begin{algorithm2e}[h]
	\SetAlgorithmName{Alg.}{}
	
	\caption{ADMM with linearizing applied to Eq.~\eqref{Eq_Extended_Formulation} }
	 \setstretch{1.1}
	\label{Alg_ADMMforDenoising}	
		\begin{small}
	\textbf{Input:} $X$, $\lambda_w$, $\lambda_z$, $\lambda_s$, $\rho_x$, $\rho_z$\;
	$L\gets \text{median}(X)$\quad $W \gets 0$  \quad $Z \gets 0$ \quad $E \gets 0$ \quad $U_x\gets 0$ \quad $U_z \gets 0$\;
	\While{not converged}{
		\Comment{Singular Value Thresholding  \eqref{Eq_updateL_P2}}
		$\displaystyle L \gets \arg \min_{{\bf A}} \widehat{\mathcal{L}}_L({\bf A}, W,Z,E,U_x,U_z)$ 
		
		\Comment{FFT \& diag solve \& iFFT \eqref{Eq_updateW_P2}}
		$\displaystyle W \gets \arg \min_{{\bf A}} \widehat{\mathcal{L}}_W (L, {\bf A},Z,E,U_x,U_z)$
		
		\Comment{Element-wise clipping \eqref{Eq_updateZ_P2}}
		$\displaystyle Z \gets \arg \min_{{\bf A}} \mathcal{L} (L,W,\bA,E,U_x,U_z)$
		
		\Comment{Soft-thresholding \eqref{Eq_updateS_P2}}
		$\displaystyle E \gets \arg \min_{{\bf A}} \mathcal{L} (L,W,Z,\bA,U_x,U_z)$
		
		\Comment{Element-wise mult and add \eqref{Eq_updateUx_P2}}
		$U_x \gets U_x + \rho_x \left((1-W) \circ \left(L-X\right)+E\right)$\;
		
		\Comment{Element-wise add and subtract \eqref{Eq_updateUz_P2}}
		$U_z \gets U_z + \rho_z \left(Z-D_{\text{3D}}(W)\right)$\;
	}
	\textbf{Output:} $L,~W,~E$\;
		\end{small}
\end{algorithm2e}

\subsection{Computational Complexity}
In this part, we briefly discuss the computational complexity of the proposed algorithms. The most computationally demanding step of the Algorithm \ref{Alg_ADMMforMRPCA} is the singular value thresholding step which is the same as the original RPCA via principal component pursuit. In recent years, variety of different methods have been developed to reduce the computation cost of the  PCP algorithm. For instance, the computation of the SVD can be reduced using the power method \cite{pope2011real}. We would like to note that all these methods are also applicable to the proposed algorithms in this study. For Algorithm \ref{Alg_ADMMforDenoising}, in addition to the singular value thresholding computation of the Fourier transform is also required which can be done in $\mathcal{O}(mn\log(mn))$ time using the Fast Fourier Transform. 

\section{Experimental Results and Discussion}
In this section, we present several experiments and comparisons on real-world video sequences with a variety of scenarios. For the experiments with static background we use the Baseline category form the Change Detection (CDnet) dataset \cite{ChangeDetection} and for the dynamic background case we use the I2R dataset \cite{I2Rdataset}. In our evaluations we shall use both statistical measures and visual comparisons. Considering the number of true positive (tp), true negative (tn), false positive (fp), and false negative (fn), the recall (Re = tp/(tp+fn)), Precision (Pre= tp/(tp+fp)), and F-measure (F1 = $2 \ \frac{\text{Pre x Re}}{\text{Pre+Re}}$) are employed as quantitative metrics  to evaluate the performance of the foreground detection. 

\subsection{Results for Static Background }
In our first experiment we consider the video sequences with static background and investigate different settings for the M-RPCA as in \eqref{Eq_Main_Formulation}, EM-RPCA as in \eqref{Eq_Extended_Formulation}, and RPCA. Here, we use 100 frames from different challenging parts of the Baseline videos in CDnet. 
The qualitative and quantitative results are presented in \Fig \ref{Fig_CompMvsRPCA} and Table 1 respectively. 
As we can observe, increase in $\lambda_w$ for M-RPCA results in increase in precision and decrease in recall (similar trade-off is present for RPCA w.r.t. the threshold value). M-RPCA shows improvement over RPCA in general while keeping the recall and precision at relatively high levels.

\begin{figure}[h]
	\begin{center}
		\centerline{\includegraphics[width=\linewidth]{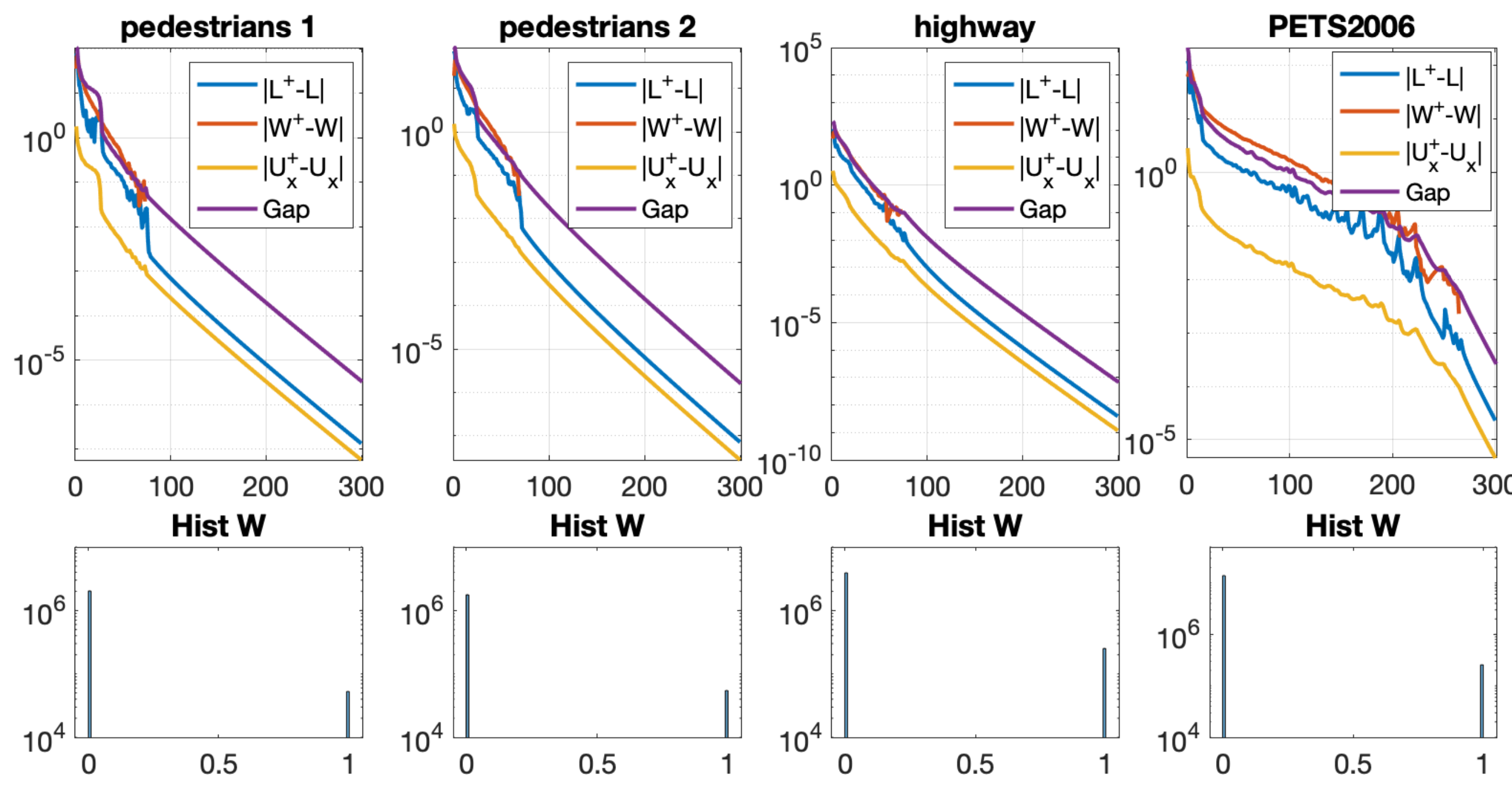}}
		\caption{Empirical convergence results of the Algorithm \ref{Alg_ADMMforMRPCA} for four different datasets.}
		\label{Fig_ConvergeRPCA}
	\end{center}
	\vskip -0.5in
\end{figure}

As we can see from the results of EM-RPCA in Table 1, considering the spacial connectivity increases the performance in the presence of multiple and overlapping objects. 
For the PETS2006, since the man in the back is fairly stationary in the entire  sequence, enforcing the connectivity results in more zeros in the mask of that region. 

The convergence properties of the Algorithm \ref{Alg_ADMMforMRPCA} for different datasets is shown in \Fig \ref{Fig_ConvergeRPCA}. We can observe that the variables have converged and the $\|(1-W)\circ(X-L)\|$ (denoted as Gap) converges to zero which shows the satisfaction of the assumptions in convergence analysis as well as feasibility of the final result. The histogram of the resulting $W$ for M-RPCA, shown in the second row, illustrates the fact that in all the cases the binary mask is directly recovered. 

\begin{figure*}[]
	\begin{center}
		\centerline{\includegraphics[width=\linewidth]{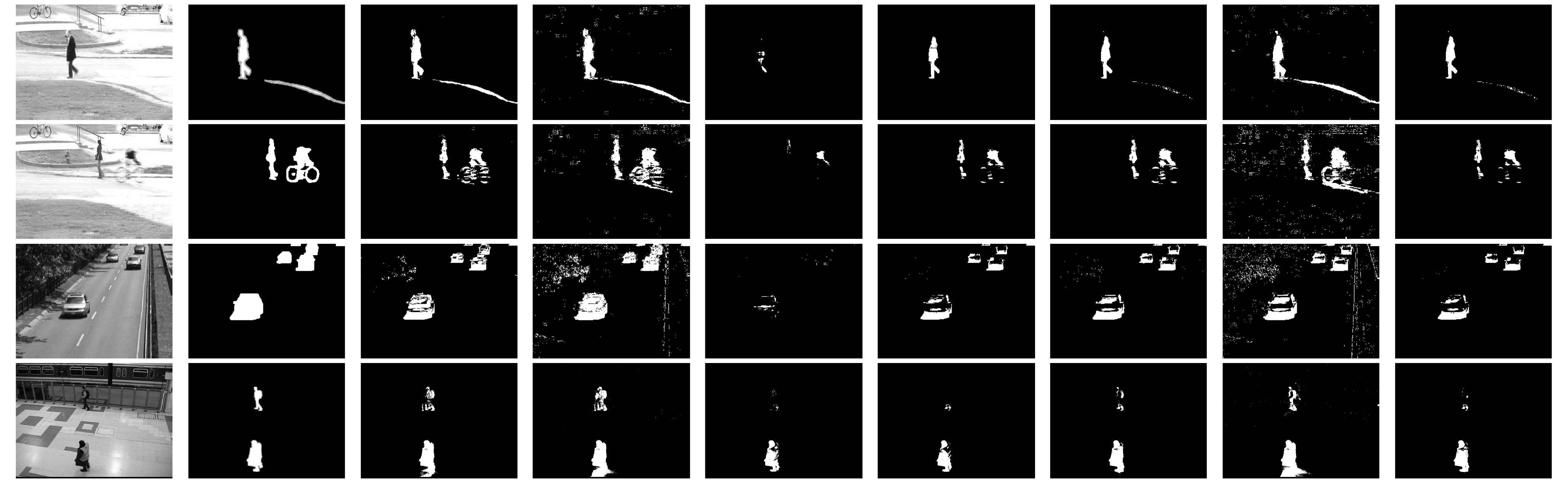}}
		\caption{Effect of parameters. columns l to r: original frame, ground truth, moderate, high, and low values of $\lambda_w$ in M-RPCA, low to high threshold for RPCA, and RPCA thresholded with Otsu method. Rows t to b: Pedestrian1, Pedestrian2, Highway, and PETS2006 }
		\label{Fig_CompMvsRPCA}
	\end{center}
	\vskip -0.2in
\end{figure*}

\begin{table*}[]
	\label{TB_compMvsRPCA}
\caption{Quantitative comparison of M-RPCA \eqref{Eq_Main_Formulation}, EM-RPCA \eqref{Eq_Extended_Formulation}, and RPCA. for static background. The parameter for M-RPCA $\{\lambda_w\}$ , RPCA (threshold value to obtain the mask), and EM-RPCA $\{\lambda_w,~\lambda_z,~\lambda_e\}$ are indicated in the "PARAM" columns}
	\begin{center}
		\begin{scriptsize}
			\begin{sc}
\begin{tabular}{ccccccccccccccccc}
	\hline
	& \multicolumn{4}{c}{pedestrians 1} & \multicolumn{4}{c}{pedestrians 2} & \multicolumn{4}{c}{highway} & \multicolumn{4}{c}{PETS2006} \\ \hline
	& param & Re & Pre & F1 & param & Re & Pre & F1 & param & Re & Pre & F1 & param & Re & Pre & F1 \\ \hline
	\multirow{3}{*}{M-RPCA} & 1e-3 & 0.97 & 0.97 & \textbf{0.97} & 7e-4 & 0.85 & 0.83 & 0.84 & 9e-4 & 0.7 & 0.94 & 0.80 & 5e-5 & 0.76 & 0.81 & \textbf{0.79} \\
	& 1e-4 & 0.99 & 0.67 & 0.80 & 7e-5 & 0.93 & 0.46 & 0.62 & 9e-5 & 0.92 & 0.68 & 0.78 & 5e-6 & 0.84 & 0.70 & 0.76 \\
	& 1e-2 & 0.4 & 1.00 & 0.49 & 7e-3 & 0.17 & 1.00 & 0.31 & 9e-3 & 0.03 & 0.97 & 0.05 & 5e-4 & 0.60 & 0.97 & 0.74 \\ \hline
	\multicolumn{1}{l}{\multirow{3}{*}{EM-RPCA}} & 1e-5 & \multirow{3}{*}{0.94} & \multirow{3}{*}{0.96} & \multirow{3}{*}{0.95} & 1e-5 & \multicolumn{1}{l}{\multirow{3}{*}{0.88}} & \multicolumn{1}{l}{\multirow{3}{*}{0.87}} & \multicolumn{1}{l}{\multirow{3}{*}{\textbf{0.88}}} & 1e-5 & \multicolumn{1}{l}{\multirow{3}{*}{0.94}} & \multicolumn{1}{l}{\multirow{3}{*}{0.89}} & \multicolumn{1}{l}{\multirow{3}{*}{\textbf{0.92}}} & 1e-5 & \multicolumn{1}{l}{\multirow{3}{*}{0.64}} & \multicolumn{1}{l}{\multirow{3}{*}{0.79}} & \multicolumn{1}{l}{\multirow{3}{*}{0.71}} \\
	\multicolumn{1}{l}{} & 1e-5 &  &  &  & 1e-5 & \multicolumn{1}{l}{} & \multicolumn{1}{l}{} & \multicolumn{1}{l}{} & 1e-5 & \multicolumn{1}{l}{} & \multicolumn{1}{l}{} & \multicolumn{1}{l}{} & 1e-5 & \multicolumn{1}{l}{} & \multicolumn{1}{l}{} & \multicolumn{1}{l}{} \\
	\multicolumn{1}{l}{} & 5e-3 &  &  &  & 5e-3 & \multicolumn{1}{l}{} & \multicolumn{1}{l}{} & \multicolumn{1}{l}{} & 5e-3 & \multicolumn{1}{l}{} & \multicolumn{1}{l}{} & \multicolumn{1}{l}{} & 5e-3 & \multicolumn{1}{l}{} & \multicolumn{1}{l}{} & \multicolumn{1}{l}{} \\ \hline 
	\multirow{4}{*}{RPCA} & 0.45 & 0.84 & 1.00 & 0.91 & 0.45 & 0.59 & 0.99 & 0.74 & 0.35 & 0.51 & 0.98 & 0.67 & 0.35 & 0.54 & 0.98 & 0.70 \\
	& 0.5 & 0.88 & 1.00 & 0.93 & 0.5 & 0.67 & 0.95 & 0.80 & 0.4 & 0.59 & 0.95 & 0.73 & 0.45 & 0.65 & 0.93 & 0.76 \\
	& 0.55 & 0.97 & 0.69 & 0.81 & 0.55 & 0.80 & 0.32 & 0.46 & 0.45 & 0.66 & 0.63 & 0.64 & 0.55 & 0.80 & 0.55 & 0.65 \\
	& Otsu & 0.87 & 1.00 & 0.93 & Otsu & 0.58 & 0.99 & 0.73 & Otsu & 0.49 & 0.98 & 0.66 & Otsu & 0.60 & 0.95 & 0.74 \\ \hline
\end{tabular}
\end{sc}
\end{scriptsize}
\end{center}
\vskip -0.1in
\end{table*}

\begin{table*}[]
\caption{Comparison of M-RPCA with other methods over challenging videos from CDnet. }
\label{TB_Compdifferentmeth}
\begin{center}
\begin{small}
\begin{sc}
\begin{tabular}{cccccccccc}
	\hline
	\multicolumn{1}{l}{}       & \multicolumn{3}{c}{Shade}                     & \multicolumn{3}{c}{Office}                    & \multicolumn{3}{c}{Winter}                    \\ \hline
	\multicolumn{1}{l}{}       & Re            & Pre           & F1            & Re            & Pre           & F1            & Re            & Pre           & F1            \\ \hline
	\multicolumn{1}{l}{\quad \quad \quad \quad \quad \quad \quad \textbf{M-RPCA}}    & \textbf{0.77} & \textbf{0.80} & \textbf{0.79} & 0.68          & \textbf{0.83} & \textbf{0.74} & 0.59          & 0.43          & 0.50          \\
	\textbf{Decolor} \cite{DECOLOR}                    & 0.73          & 0.32          & 0.42          & \textbf{0.87} & 0.61          & 0.71          & \textbf{0.64} & \textbf{0.70} & \textbf{0.69} \\
	\textbf{GMM} \cite{GMM}                       & 0.75          & 0.71          & 0.72          & 0.53          & 0.82          & 0.59          & 0.39          & 0.58          & 0.45          \\
\textbf{FBM} \cite{FBM}                        & 0.65          & 0.77          & 0.70          & 0.62          & 0.76          & 0.71          & 0.37          & 0.36          & 0.34          \\
	\textbf{RPCA} \cite{OriginalRPCA}                       & 0.69          & 0.74          & 0.71          & 0.57          & 0.76          & 0.62          & 0.55          & 0.40          & 0.42          \\
	\textbf{ViBe} \cite{ViBe}                      & 0.74          & 0.78          & 0.76          & 0.7           & 0.8           & 0.69          & 0.57          & 0.18          & 0.23          \\
	\textbf{SOBS} \cite{SOBs}                      & 0.63          & 0.78          & 0.68          & 0.67          & 0.79          & 0.69          & 0.18          & 0.53          & 0.24          \\ \hline
\end{tabular}
\end{sc}
\end{small}
\end{center}
\end{table*}

\begin{figure*}[]
	\begin{center}
		\centerline{\includegraphics[width=\linewidth]{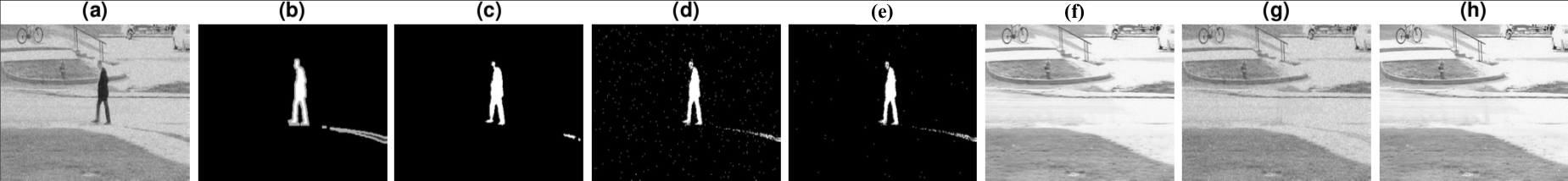}}
		\caption{Visual results for synthetic noise. (a) noisy frame, (b) Ground truth, (c,f) EM-RPCA, (d,g) M-RPCA, and (e,h) RPCA.  }
		\label{Fig_DenoisingvsRPCA} 
	\end{center}
	\vskip -0.2in
\end{figure*}

\begin{figure*}[]
	\begin{center}
		\centerline{\includegraphics[width=\linewidth]{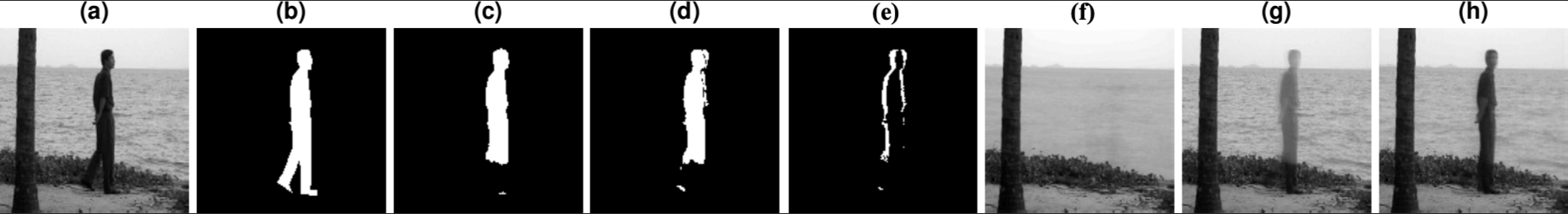}}
		\caption{Visual results for WaterSurface: (a) original frame (b) Ground truth, mask from (c) EM-RPCA  (d) TVRPCA (e) RPCA, low-rank image from (f) EM-RPCA  (g) TVRPCA (h) RPCA. }
		\label{Fig_Dynamicwatersurf}
	\end{center}
	\vskip -0.2in
\end{figure*}

We additionally compare the result of M-RPCA to other methods in the literature over other challenging sequences from CDnet. The experiments in this case closely follow the experimental setup in \cite{ExperimetsforMRPCA}. 
As it can be seen from  Table 2, the M-RPCA outperforms the other RPCA based methods. In the Winter sequence, due to the artifact present, M-RPCA is outperformed by DECOLOR \cite{DECOLOR}.

\subsection{Results for the Extended M-RPCA and Dynamic Background}
In this section we will show experimental results using the EM-RPCA as in \eqref{Eq_Extended_Formulation} with dynamic background videos. First, we show the robustness of EM-RPCA to random noise, by adding synthetic noise to stationary background sequence and evaluate the performance of different methods. 
The results of our model compared with RPCA are shown in  Table \ref{EM-RPCAvsRPCA artificial noise} and \Fig \ref{Fig_DenoisingvsRPCA}. As it can be seen EM-RPCA achieves good performance in terms of peak signal-to-noise ratio (PSNR) of the recovered background and the F-measure of the recovered mask.
\begin{table}[]
	\caption{Comparison of different methods with synthetic noise }
	\label{EM-RPCAvsRPCA artificial noise}
	\vskip 0.15in
	\begin{center}
		\begin{small}
			\begin{sc}
				\begin{tabular}{ccccc}
					\hline
					(SNR = $7.7dB$)	& \multicolumn{2}{c}{Pedestrian} & \multicolumn{2}{c}{Highway}  \\
					& F1            & PSNR           & F1           & PSNR          \\ \hline
					EM-RPCA & 0.94          & 34.65          & 0.90         & 30.80            \\
					M-RPCA  & 0.84          & 24.43          & 0.69         & 24.32            \\
					RPCA    & 0.90          & 34.10          & 0.60         & 31.14            \\ \hline
				\end{tabular}
			\end{sc}
		\end{small}
	\end{center}
	\vskip -0.1in
\end{table}

In order to evaluate the performance of the EM-RPCA, dynamic background sequences from CDnet and I2R dataset are used. As an example, Figure \ref{Fig_components} shows the low-rank, foreground mask and dynamic background components recovered by EM-RPCA. We compare our results to the RPCA and TVRPCA \cite{TVRPCA}. TVRPCA  generalizes RPCA for the dynamic background cases by decomposing the sparse component from the RPCA model into changing background and specially connected foreground. The EM-RPCA enforces the connectivity of the foreground based on the overlaying model.  Table \ref{EM-RPCAvsRPCAvsTV} shows the F-measures for different video sequences. TVRPCA and EM-RPCA perform similarly in terms of this measure. 
\begin{table}[h]
	\caption{Comparison of EM-RPCA with different methods on dynamic background.}
	\label{EM-RPCAvsRPCAvsTV}
	\begin{center}
		\begin{small}
			\begin{sc}
				\begin{tabular}{cccc}
 & EM-RPCA & TVRPCA & RPCA \\ \hline
WaterSurface & \textbf{0.88} & \textbf{0.88} & 0.41 \\
Fountain & \textbf{0.81} & 0.80 & 0.57 \\
Campus & \textbf{0.77} & \textbf{0.77} & 0.72 \\
Fountain 2 & 0.71 & \textbf{0.72} & 0.43 \\
Overpass & \textbf{0.78} & 0.77 & 0.46 \\ \hline
\end{tabular}
			\end{sc}
		\end{small}
	\end{center}
	\vskip -0.1in
\end{table}

Visual results in \Fig\ref{Fig_Dynamicwatersurf} and \Fig\ref{Fig_OfficeSlow} on the other hand show that overlaying model is better capable of recovering the background in case where the foreground stops moving. The ROC curve and the histogram of the $W$ are shown in \Fig \ref{Fig_DynamicMeasures}. The ROC curve shows slight improvement in terms of area under the curve. Additional visual results for another sequence are shown in Figures \ref{Fig_CDnetDynmasks} and \ref{Fig_CDnetDy}. 
\begin{figure}[]
	\begin{center}
		\centerline{\includegraphics[width=\linewidth]{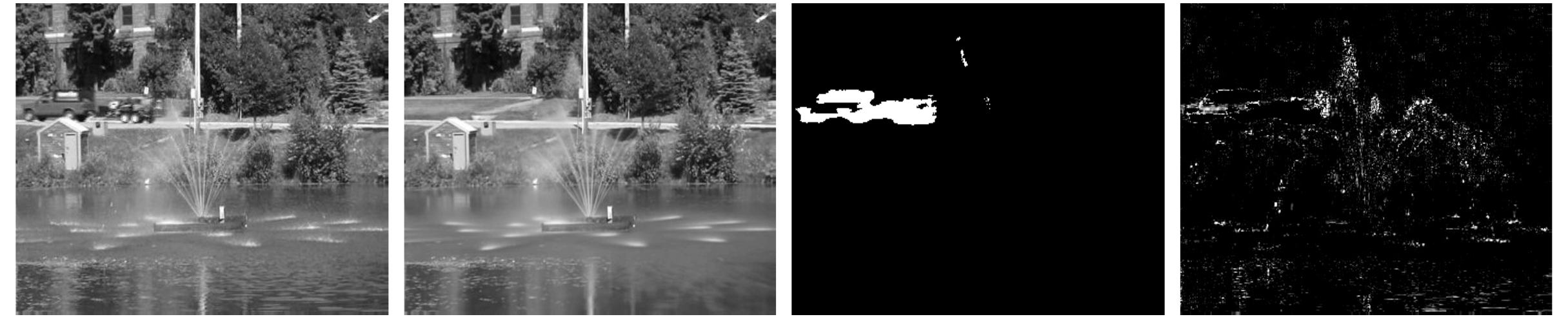}}
		\caption{Different components by EM-RPCA form left to right: Original frame, low-rank component $L$, foreground mask $W$, and dynamic changes of background $E$.}
		\label{Fig_components}
	\end{center}
	\vskip -0.2in
\end{figure}

\begin{figure}[]
	\begin{center}
		\centerline{\includegraphics[width=\linewidth]{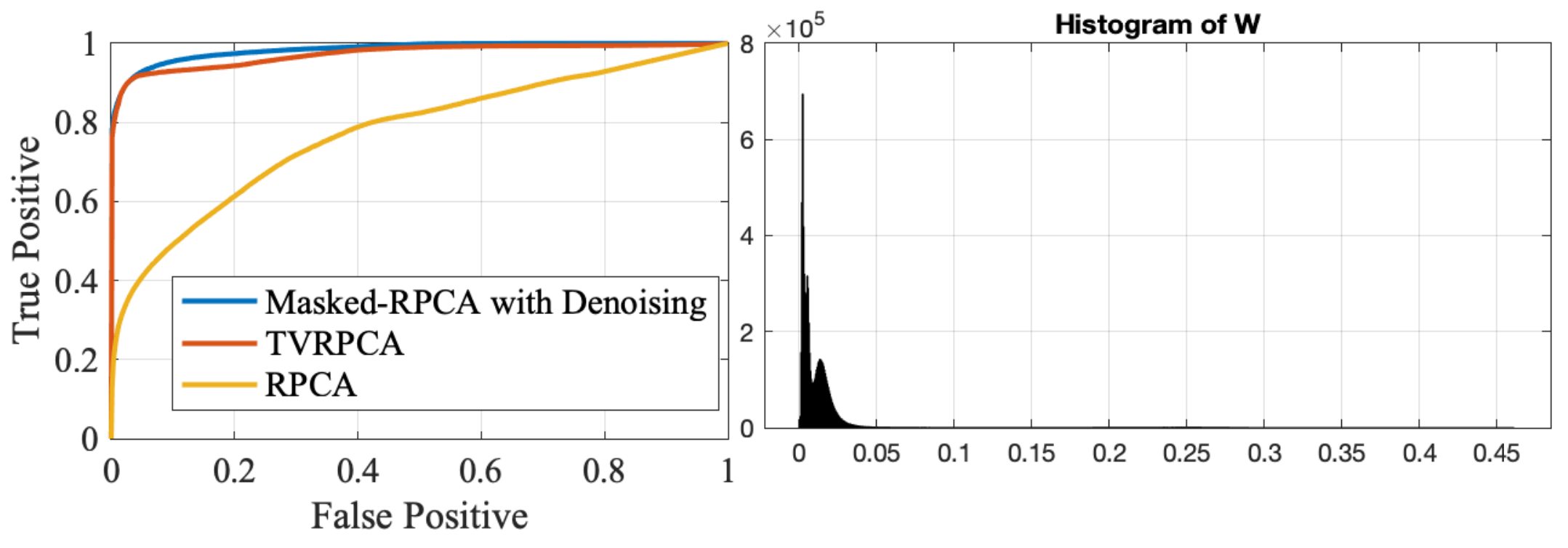}}
		\caption{(left) ROC curve for the result of different methods on WaterSurface data. (right) Histogram of the recovered $W$ from EM-RPCA.}
		\label{Fig_DynamicMeasures}
	\end{center}
	\vskip -0.2in
\end{figure}
\begin{figure}[]
	\begin{center}
		\centerline{\includegraphics[width=\linewidth]{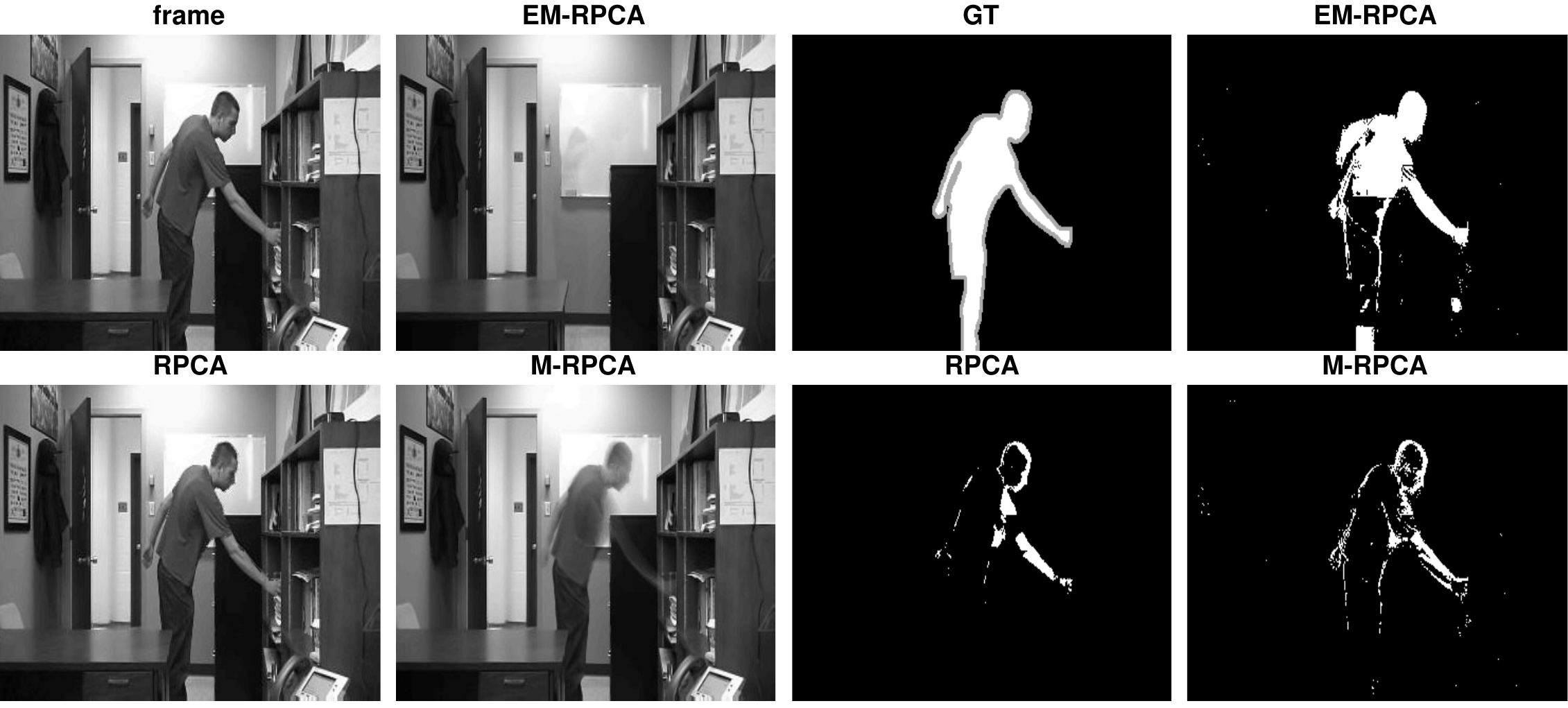}}
		\caption{Visual results for the mask and the low-rank background for the slow moving sequence.  }
		\label{Fig_OfficeSlow}
	\end{center}
	\vskip -0.2in
\end{figure}

\begin{figure}[]
	\begin{center}
		\centerline{\includegraphics[width=\linewidth]{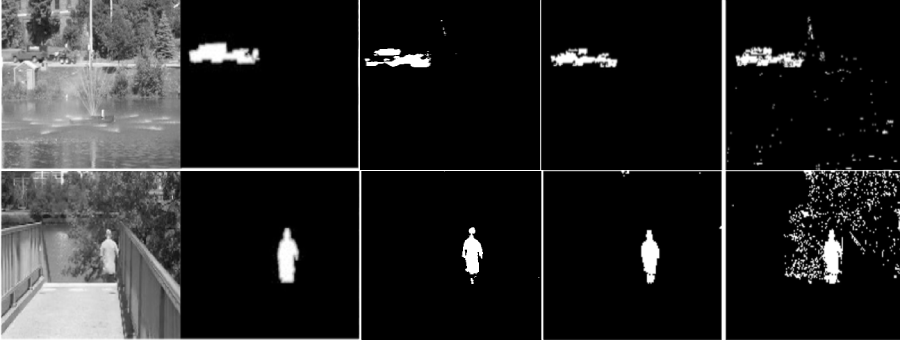}}
		\caption{Visual results for the mask for Fountain2(top row) and Overpass(bottom row) datasets from CDnet. Images form left to right: Original frame, Ground truth mask, EM-RPCA mask, TVRPCA mask, and RPCA mask.}
		\label{Fig_CDnetDynmasks}
	\end{center}
	\vskip -0.2in
\end{figure}
\begin{figure}[]
	\begin{center}
		\centerline{\includegraphics[width=\linewidth]{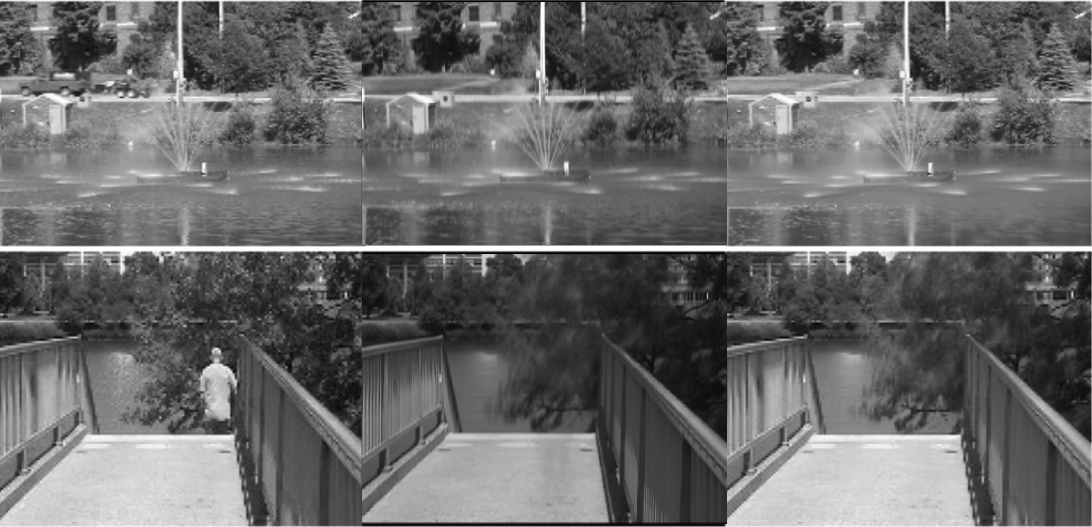}}
		\caption{Visual results for the low-rank image for Fountain2(top row) and Overpass(bottom row) datasets from CDnet. Images form left to right: Original frame, low-rank component for EM-RPCA, and low-rank component for TVRPCA.}
		\label{Fig_CDnetDy}
	\end{center}
	\vskip -0.2in
\end{figure}

\section{Conclusion }
In this study, we introduced an extension of sparse and low-rank decomposition under overlaying model, and developed an optimization framework to solve it.
We also propose an extension of our M-RPCA framework for the dynamic background.
We also provide an analysis of the model convergence under reasonable assumptions.
We performed an extensive experimental studies evaluating our model on multiple videos from CDnet and I2R datasets, and
 show improvements for both  static and dynamic background over RPCA and its extensions.
As  future work, we plan to extend this framework for various  scenarios such as  camera jitter.


%
\clearpage
\begin{small}

\end{small}

\end{document}